\documentclass{article}

\usepackage[english]{babel}

\usepackage[a4paper,top=2cm,bottom=2cm,left=3cm,right=3cm,marginparwidth=1.75cm]{geometry}

\usepackage{amsmath}
\usepackage{graphicx}
\usepackage[colorlinks=true, allcolors=blue]{hyperref}

\usepackage{amsthm}
\newtheorem{theorem}{Theorem}
\newtheorem{lemma}{Lemma}

\newtheorem{corollary}{Corollary}[theorem]
\newtheorem*{definition}{Definition} 

\usepackage{authblk}

\title{Enhancing Conformal Prediction Using E-Test Statistics}

\author[1]{A.A.~Balinsky}
\author[2]{A.D.~Balinsky}
\affil[1]{Department of Mathematics, Cardiff University, UK}
\affil[2]{Department of Computing,
Imperial College London, UK}

\date{}

\begin{document}
\maketitle

\begin{abstract}
 Conformal Prediction (CP) serves as a robust framework that quantifies uncertainty in predictions made by Machine Learning (ML) models. Unlike traditional point predictors, CP generates statistically valid prediction regions, also known as prediction intervals, based on the assumption of data exchangeability. Typically, the construction of conformal predictions hinges on p-values. This paper, however, ventures down an alternative path, harnessing the power of e-test statistics to augment the efficacy of conformal predictions
 by introducing a \textit{BB-predictor} (bounded from the below predictor).
\end{abstract}

\section{Introduction}

We will assume that the reader is familiar with the basics of CP,
see \cite{10.3150/21-BEJ1447, angelopoulos2023conformal} for recent reviews and introductions. 

The first key benefit of CP is that it is \textit{distribution-free}: CP doesn't require assumptions about the underlying data distribution, making it adaptable to diverse scenarios. As such, CP is usually based on 
rank-based statistics and permutations. The second key benefit is that CP can be built under the \textit{exchangeability} assumption, which is weaker than the i.i.d. assumption.
Recall that the exchangeability of
$L_1, \ldots, L_{n+1}$ means that their joint distribution is unchanged under permutations:
\[ (L_1, \ldots, L_{n+1}) \stackrel{\text{d}}{=} (L_{\sigma(1)}, \ldots, L_{\sigma(n+1)})\]
 for all permutations $\sigma$.

The next Lemma is the main mathematical tool of standard CP theory and is based on
the idea of the p-value from hypothesis testing in statistics:

\begin{lemma} \label{CP_lemma}
    Suppose $L_1, \ldots, L_{n+1}$ are exchangeable random variables. Set
    $$
    U = \# \{ i = 1, \ldots, n+1 : \ L_i \geq L_{n+1}\},
    $$
    i.e., the number of $L$ that are at least as large as the last one. Then
    \[
      P \left\{ \frac{U}{n+1} > \epsilon \right\} \geq 1-\epsilon
    \]
    for all $\epsilon\in [0,1]$
\end{lemma}

% State the specific problem you are addressing and its importance.

In this article, we would like to develop another approach to CP by using 
the theory of hypothesis testing based on   
\textit{e-test statistics},
\cite{safe_test_arxiv}. The idea behind e-test statistics is very simple and is a straightforward application of 
Markov’s inequality: if $E$ is non-negative random variables
with the expectation $\mathbf{E}(E) \leq 1$, then
\[
P(E \geq 1/\alpha) \leq \alpha,
\]
for any positive $\alpha$. If, for example, $\alpha = 0.05$, then the event $E \geq 20$ will have probability less than 5\%. Though elementary, Markov's inequality in combination with Cramer-Chernoff's method turns out to be very powerful and surprisingly sharp \cite{Massart2007l}.

In the case of an unknown probability distribution, evaluating the mean value of a random variable becomes challenging, especially if we lack an upper bound estimate for this variable. Such an estimate is necessary to normalize the random variable before applying Markov's inequality. Naturally, as often occurs in Machine Learning, we can approximate the mean value empirically if we have a sufficient number of samples. However, we cannot provide a guarantee similar to Lemma~\ref{CP_lemma}. The following section will present our main theoretical result (Theorem~\ref{main_result}) on addressing this issue in the context of exchangeable random variables.

\section{Main Result}
The main theoretical result of this article is the following theorem

\begin{theorem} \label{main_result}
Suppose $L_1, \ldots, L_{n+1}$ are exchangeable non-negative random variables. Set
    $$
    F = \frac{L_{n+1}}{\left(\sum\limits_{j=1}^{n+1} L_j \right) /(n+1)}.
    $$
Then the expectation $\mathbf{E}(F) = 1$ and
 \[
      P \{F \geq 1/\alpha\} \leq \alpha,
    \]
    for any positive $\alpha$.
\end{theorem}

\begin{proof}
Let us introduce the following random variables
$$
F_i = \frac{L_{i}}{\left(\sum\limits_{j=1}^{n+1} L_j \right) /(n+1)}, \ \ \ \ i=1,2,\ldots, n+1.
$$
Due to the exchangeability of $L_1, \ldots, L_{n+1}$, all random variables $F_i, \ i=1, \ldots, n+1$ are identically distributed, so they all have the same expectation $\mathbf{E}(F_i)$. 
\[
\mathbf{E}(F_1 + F_2 + \ldots + F_{n+1}) = (n+1)\times \mathbf{E}(F_{n+1}) = (n+1)\times \mathbf{E}(F).
\]
From  definition of $F_i$ we can see that $F_1 + F_2 + \ldots + F_{n+1} = n+1$. Thus, $(n+1)\times \mathbf{E}(F) = n+1$, and 
$\mathbf{E}(F) = 1$.  Markov's inequality implies that 
\[
      P \{F \geq 1/\alpha\} \leq \alpha
    \]
    for any positive $\alpha$.
\end{proof}

Now, $F \geq 1/\alpha$ is equivalent to
$$
(n+1) L_{n+1} \geq  \frac{1}{\alpha} ( L_1 + \ldots L_n + L_{n+1}),
$$

$$
(\alpha (n+1) - 1) L_{n+1} \geq L_1 + \ldots + L_n,
$$

$$
\left(\alpha + \frac{\alpha -1 }{n}\right) L_{n+1} \geq \frac{L_1 + \ldots + L_n}{n}).
$$

\begin{corollary}
    Suppose $L_1, \ldots, L_{n+1}$ are exchangeable non-negative random variables. Then for any positive $\alpha$

 \begin{equation} \label{main_ineq}
     P \left\{  \frac{L_{n+1}}{\frac{L_1 + \ldots + L_n}{n}}  \geq
     \frac{1}{\alpha} \left(  \frac{1}{1+ \frac{1-1/\alpha}{n}}\right)\right\}
     \leq \alpha.
 \end{equation}
\end{corollary}

So, we can now introduce our new \textit{BB-predictor} (bounded from the below predictor)

\begin{equation} \label{bb_ineq}
     P \left\{  L_{n+1} \geq 
     \frac{1}{\alpha} \left(  \frac{1}{1+ \frac{1-1/\alpha}{n}}\right) \times {\frac{L_1 + \ldots + L_n}{n}}  \right\}
     \leq \alpha.
 \end{equation}

Up until now, our presentation has been largely theoretical. 
We have yet to provide examples of interesting exchangeable random variables. 
In the next section, we elucidate how machine learning is a source of such examples. 
We also delve into the renowned MNIST dataset to observe how CP performs in this scenario.

\section{Experiments and Results}\label{results}

In Conformal Prediction, a \textit{non-conformity measure} is a function that assigns a score to a data sample, reflecting its degree of conformity with other data samples. The pursuit of quantifying similarity and conformity is of utmost importance in Machine Learning (ML). To this end, we utilize ML techniques to construct a robust measure of similarity. Following this, we use a calibration set to enable instance-based modelling. The amalgamation of these two paradigms—model-based and instance-based—results in Inductive Conformal Prediction. This versatile framework provides statistically valid prediction regions.

Let’s examine the case of supervised machine learning through the lens of hypothesis testing theory. Our null hypothesis posits that an observation $z=(x,y)=(sample,label)$ originates from an unknown probability distribution of \textit{correct} data. To reject this null hypothesis (i.e., to declare the label incorrect), we require a \textit{statistic} that increases for incorrect labels and decreases for correct ones. This requirement closely mirrors the concept of a \textit{loss function} in Machine Learning (ML). Such a loss function can be derived from ML algorithms that learn through optimization. Let’s delve deeper into this with some specific examples.

\begin{definition}
    Let $Z$ be a set of possible observations.
    A \textbf{non-conformity measure} $A$ is a sequence of equivariant maps $A^{(n+1)}$ from $Z^{n+1}$ to $\mathbf{R}^{n+1}$ with $n =0, 1,2, \ldots$.
\end{definition}

Given a non-conformity measure $A$ and a sequence of observations $z_1, z_2, \ldots, z_{n+1}$, the output will be a sequence of non-conformity scores $\alpha_1, \alpha_2, \ldots, \alpha_{n+1}$.
\textbf{Intuition}: A small non-conformity score means that the sample is similar to other samples, i.e. from the same distribution.

In plain English, the definition above tell us, that
$$
\alpha_i = A_{n+1}(z_i; z_1, \ldots , z_{i-1}, z_{i+1}, \ldots, z_{n+1}),
$$
where $A_{n+1}: Z^{n+1} \rightarrow \mathbf{R}$ and is invariant under permutation of the last $n$ variables.

%%%%%%%%%%%%%%
%%%%%%%%%%%%%%

Suppose we equip $Z^{n+1}$ with permutation invariant probability measures. In this case, we can utilize the sequence of non-conformity scores $\alpha_1, \alpha_2, \ldots, \alpha_{n+1}$ as exchangeable random variables, as per Lemma~\ref{CP_lemma} and Theorem~\ref{main_result}. It is intuitively expected that $\alpha$ will be small for data of interest.

Let's consider the simplest possible scenario where the function $A_{n+1}$ depends solely on the first argument, that is, $\alpha_i = LossFunction(z_i)$.

For our numerical experiments, we will employ the MNIST (Modified National Institute of Standards and Technology) database of handwritten digits. This database is a popular choice for training and testing various image processing systems and machine learning models. It comprises 60,000 training images and 10,000 testing images. Each image is a grayscale representation of a handwritten digit, sized at $28 \times 28$ pixels. Let's examine some examples:

\begin{figure}[h]
\centering
\includegraphics[width=0.5\linewidth]{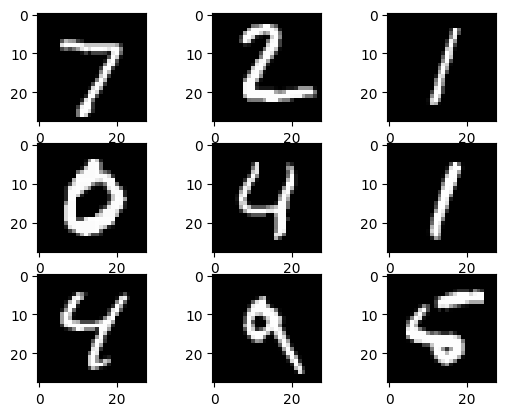}
\caption{\label{fig:samples} The first nine images from the MNIST database of handwritten digits.}
\end{figure}

We will train a very simple TensorFlow  Keras model  Figure~\ref{fig:model }
\begin{figure}[h]
\centering
\includegraphics[width=0.5\linewidth]{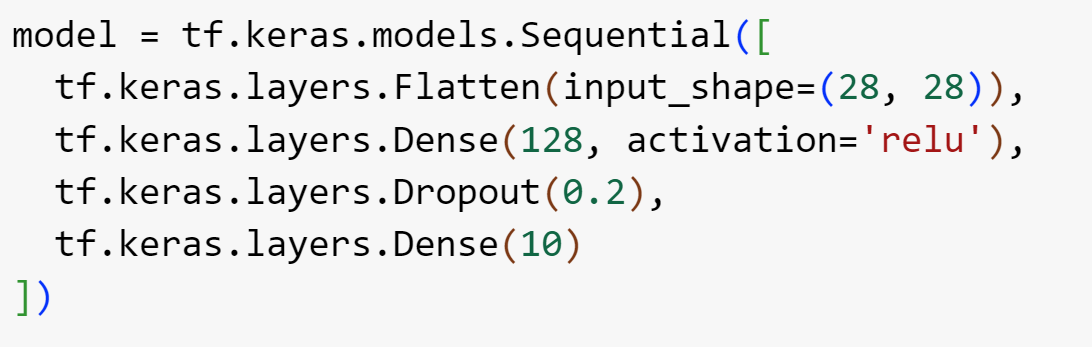}
\caption{\label{fig:model } Our model.}
\end{figure}
with \textit{SparseCategoricalCrossentropy} loss function and '\textit{adam}' optimizer.
For reproducibility we use  $tf.random.set\_seed(2024)$.

Upon training the model with the training images and labels over $25$ epochs, we achieved an accuracy of $99.22$\% on the training data and $98.17$\% on the test data. Thus, the model appears to be quite effective. Let’s examine some instances where the predictions were incorrect:
Figure~\ref{fig:wrong}
\begin{figure}[h]
\centering
\includegraphics[width=0.5\linewidth]{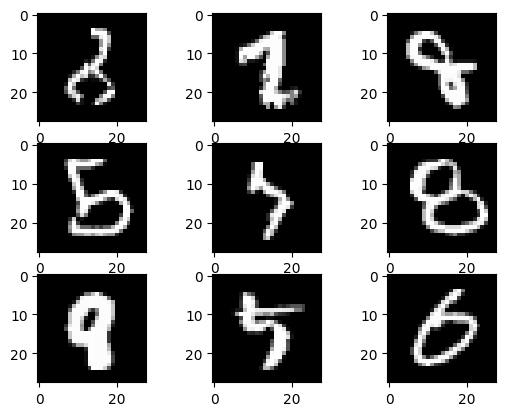}
\caption{\label{fig:wrong} Some images that the model predicts the wrong labels.}
\end{figure}

Once the model is trained and the loss function is established, we can shift our focus away from the training data and concentrate solely on the test data. We partition the test data randomly into two subsets: the CalibrationSet, which comprises 50\% of the test data, and the ConformalPredictionTestSet, which also makes up 50\% of the test data. For the CalibrationSet, a plot of the sorted values of the loss function is depicted in Figure~\ref{fig:calibr_values}.

\begin{figure}[h]
\centering
\includegraphics[width=0.5\linewidth]{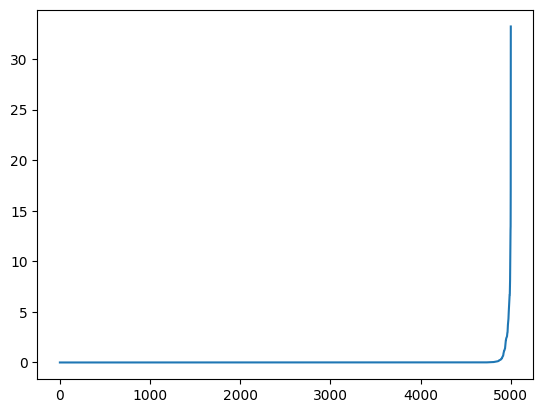}
\caption{\label{fig:calibr_values} Values of the loss function on the CalibrationSet.}
\end{figure}

\subsection{Inductive Conformal Prediction with e-test statistics }
 With $\alpha = 0.05$, $n=5000$, the \textit{BB-predictor} (\ref{bb_ineq}) tells us that for an (image, label) with the image from from the ConformalPredictionTestSet.
 $$
 P\{ LossFunction \geq  1.5002\} < 0.05.
 $$
Applying this criterion to the ConformalPredictionTestSet, we obtain that with probability more than 95\%
\begin{itemize}
    \item Examples without labels:   0
\item Examples with one  label:   4944
\item Examples with more than one  label:   56
\end{itemize}

Let us demonstrate  some examples with multiple labels Figure~\ref{fig:multi}. 
\begin{figure}[h]
\centering
\includegraphics[width=0.5\linewidth]{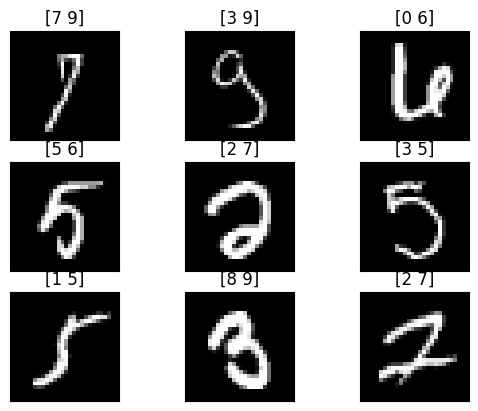}
\caption{\label{fig:multi} Some images that have multiple labels prediction with their labels.}
\end{figure}

\subsection{Inductive Conformal Prediction with p-value statistics }

Given $\epsilon = 0.05$, $n=5000$, Lemma~\ref{CP_lemma} instructs us to compute the loss of the element at index 4750 in the ordered list of losses for images from the CalibrationSet. This value is found to be 0.0247. When we apply this to the ConformalPredictionTestSet, we find that with a probability exceeding 95%,
\begin{itemize}
    \item Examples without labels:  238
\item Examples with one  label:   4762
\item Examples with more than one  label:   0
\end{itemize}

Here are some examples without labels as depicted in Figure~\ref{fig:no}:
\begin{figure}[h]
\centering
\includegraphics[width=0.5\linewidth]{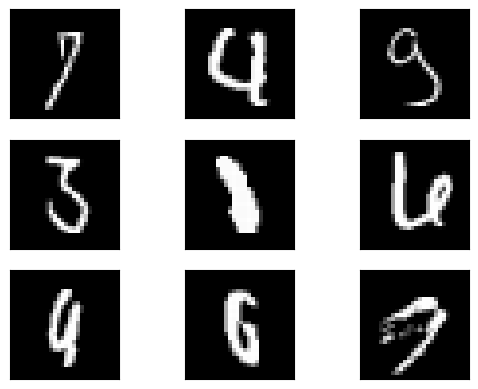}
\caption{\label{fig:no} Some images that do not have labels.}
\end{figure}

Indeed, it’s evident that the standard Inductive Conformal Prediction employing p-value statistics and our variant of Inductive Conformal Prediction using \textit{BB-predictor} (\ref{bb_ineq}) yield distinct outcomes. Furthermore, we can introduce a monotonic function to the exchangeable random variables. While this alters the mean value, it leaves the rank-based statistics unchanged. 

\section{Conclusion}
This article embarks on an exploration of Conformal Prediction (CP), a robust framework specifically designed to quantify the uncertainty inherent in machine learning predictions. Notably, CP is capable of generating statistically valid prediction regions, eliminating the need for assumptions about the data distribution.

The authors propose an enhancement to CP through the incorporation of e-test statistics and introducing a new \textit{BB-predictor}. These statistics leverage Markov’s inequality on exchangeable positive random variables, presenting a fresh perspective on the problem.

The crux of the theoretical results is the demonstration that for exchangeable non-negative random variables, a certain ratio consistently holds an expectation of 1. Moreover, the probability of this ratio can be effectively constrained using Markov’s inequality.

We underscore the importance of non-conformity measures and their crucial role in supervised learning. The MNIST dataset is employed as the experimental example, further reinforcing the concepts discussed.

%\subsection{Problem statements}

%Briefly summarize your main contributions and re-emphasize the significance of your work.

%Conclude with a forward-looking statement about the future of conformal predictions.

\bibliographystyle{alpha}

\end{document}